\documentclass{article} 
\usepackage{iclr2027_conference,times}

\usepackage{amsmath,amsfonts,bm}

\def\eqref#1{equation~\ref{#1}}

\def\1{\bm{1}}

\DeclareMathAlphabet{\mathsfit}{\encodingdefault}{\sfdefault}{m}{sl}
\SetMathAlphabet{\mathsfit}{bold}{\encodingdefault}{\sfdefault}{bx}{n}

\iclrfinalcopy

\usepackage{hyperref}
\usepackage{url}
\usepackage{microtype}
\usepackage{graphicx}
\usepackage{subcaption}
\usepackage{multirow}
\usepackage[table]{xcolor}
\usepackage{enumitem}
\usepackage{booktabs} 
\usepackage{amsmath}
\usepackage{algorithm}
\usepackage{algpseudocode}
\usepackage{xcolor}
\usepackage{wrapfig}
\usepackage{graphicx}
\usepackage{amsthm}
\usepackage{amsmath}
\usepackage{amssymb}
\usepackage{pifont}

\newcommand{\cmark}{\ding{51}}
\newcommand{\xmark}{\ding{55}}
\newcommand{\pmark}{$\triangle$}

\newtheorem{theorem}{Theorem}
\newtheorem{proposition}{Proposition}

\title{CFG-Policy: Unified Diffusion Policy Learning Framework via Classifier-free Guidance}

\author{
\centerline{
\textbf{Shutong Ding}$^{1,2,*}$ \quad
\textbf{Zhaoyang Liu}$^{1,2,*}$ \quad
\textbf{Yimiao Zhou}$^{1}$ \quad
\textbf{Jiayuan Gu}$^{1}$
}
\\[0.6em]
\centerline{
\textbf{Jingya Wang}$^{1}$ \quad
\textbf{Ye Shi}$^{1,2,\dagger}$
}
\\[0.8em]
\centerline{
$^{*}$Equal Contribution
\qquad
$^{\dagger}$Corresponding Author
}
\\[0.5em]
\centerline{
$^{1}$ShanghaiTech University, Shanghai, China $\quad$
$^{2}$InstAdapt, Shanghai, China
}
}

\begin{document}

\maketitle

\begin{abstract}
Diffusion policies are increasingly used in robotic control and manipulation, vision-language-action models, and world-model-based agents. Their ability to represent multimodal action distributions offers rich opportunities for exploration in reinforcement learning (RL). However, optimizing diffusion policies with RL can reduce the behavioral diversity and further limit their exploration capability. In RL post-training, directly fine-tuning a pretrained policy may also disrupt its acquired capabilities, while action-space residuals may lack the expressive power for multimodal exploration. Together, these issues reveal a common challenge across online and offline-to-online diffusion RL: improving task performance while retaining the behavioral diversity and capabilities needed for continued adaptation. We introduce CFG-Policy, a unified, plug-and-play framework that separates exploration and exploitation into two generative policies and combines their score or velocity fields during sampling through classifier-free guidance. For online RL, a frozen exploration branch sustains behavioral diversity while an exploitation branch learns from task rewards; for offline-to-online RL, a pretrained exploitation branch remains frozen while a trainable branch steers the composed behavior. CFG-Policy preserves the underlying RL objectives and improves performance with QSM and QVPO on 5 MuJoCo tasks and with Flow-SDE+PPO on 4 RoboTwin manipulation tasks. Project page is at \url{https://zhy-liu.github.io/cfg-policy-page}.
\end{abstract}

\section{Introduction}\label{sec:intro}

Diffusion policies have become an expressive model class for sequential decision making. Their ability to represent complex, multimodal action distributions has supported robotic control and manipulation \citep{chi2023diffusionpolicy,liu2025sample}, vision-language-action policies \citep{chen2025pirl}, and world-model-based policy adaptation \citep{chandra2025diwa}. Related flow-matching policies generate actions through continuous velocity fields and are increasingly used in reinforcement learning (RL) \citep{lipman2023flowmatching,park2025fql,ding2026genpo,hu2026genpo}. Beyond fitting demonstrations, these generative policies offer a rich space of plausible actions to explore when learning from rewards.

However, the diversity that makes diffusion policies attractive for exploration can decline as reward optimization concentrates the policy on currently successful behaviors. Prior work has observed mode collapse and training instability when diffusion models are fine-tuned with RL \citep{barcelo2024avoiding}. In online diffusion RL, methods such as Q-Score Matching and QVPO improve task performance by updating a single generative policy \citep{psenka2024qsm,ding2024qvpo}; continued improvement then depends on that same policy still producing sufficiently diverse actions. In RL post-training, direct optimization of a pretrained diffusion or flow policy \citep{ren2025dppo,zhang2025reinflow} may disturb capabilities learned before RL. Residual policy methods can preserve a base policy, but action-space corrections may provide less expressive exploration than modifying the generative process itself \citep{johannink2019residual,yuan2025policydecorator}. These settings expose the same underlying challenge: how can diffusion RL improve task performance while retaining the behavioral diversity and capabilities needed for continued adaptation?

To address these challenges, we propose a unified, plug-and-play framework called \textbf{CFG-Policy} that separates exploration and exploitation into two generative policies. Inspired by classifier-free guidance (CFG) \citep{ho2022classifier}, it combines their score or velocity fields at each step of action generation. Given two fields $v_A$ and $v_B$, the field of behavior policy is induced by
\begin{equation}
v_w = v_A + w  (v_B - v_A)
\end{equation}
where $w$ controls their relative contribution. Both fields act on the same intermediate sample, so the combination steers the entire generative trajectory rather than interpolating completed actions. Because this composition replaces only the behavior sampler, it can be added to different diffusion RL algorithms without redefining their actor or critic objectives.

\begin{figure}
    \centering
    \includegraphics[width=\linewidth]{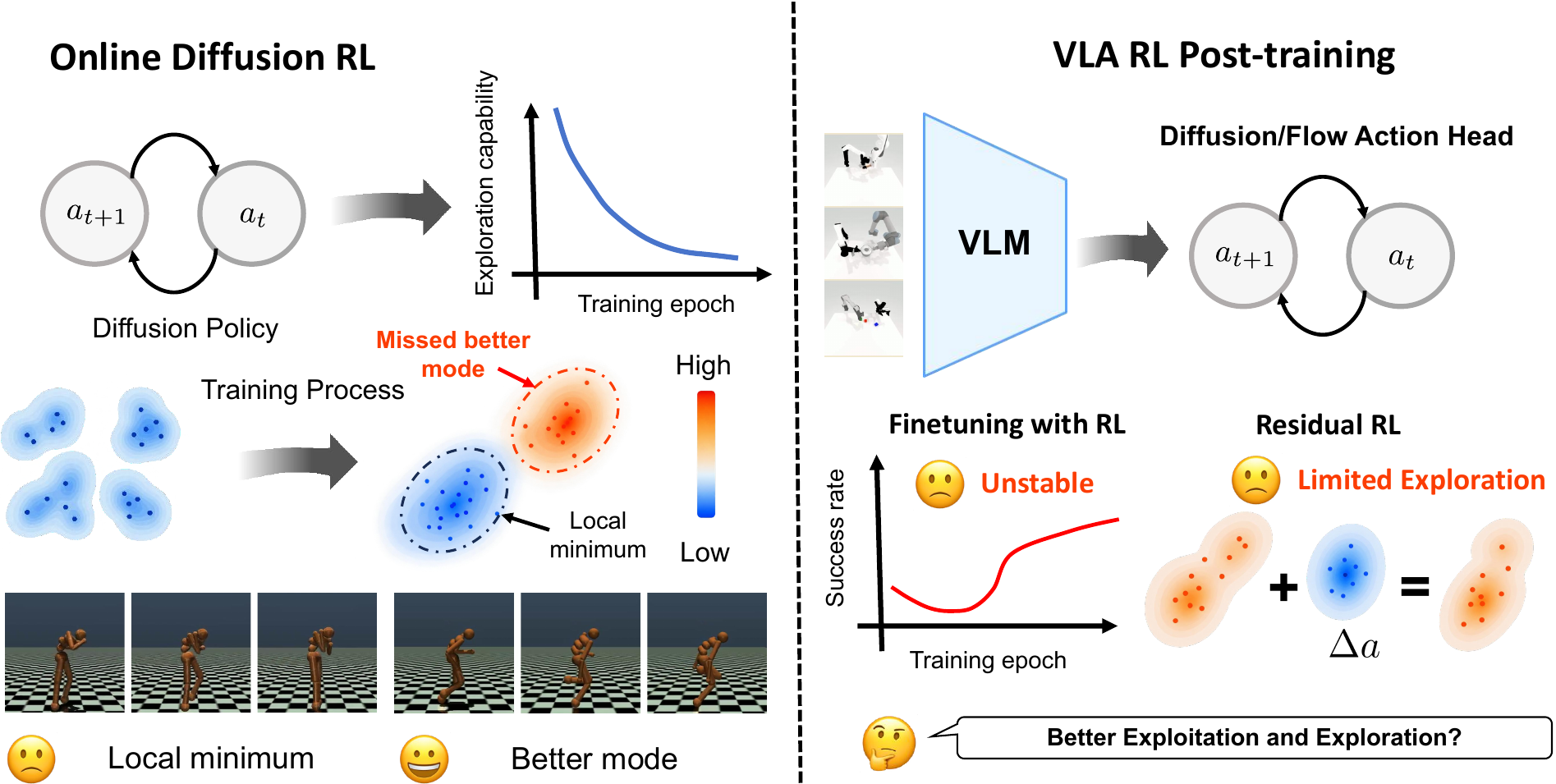}
    \caption{Exploration–exploitation challenges in diffusion RL. In online diffusion RL, optimizing a single generative policy can progressively reduce behavioral diversity and miss better modes. In VLA RL post-training, directly fine-tuning the pretrained policy may disrupt pretrained capabilities, while residual policies can provide limited exploration.}
    \label{fig:placeholder}
\end{figure}

The same mechanism handles two complementary learning regimes by changing which branch is frozen. In \textbf{online RL}, a separately trained exploration policy remains fixed while the exploitation policy learns from task rewards; the composed sampler maintains an exploratory source throughout training. In \textbf{offline-to-online RL post-training}, a pretrained exploitation policy remains fixed while a second, trainable branch steers the composed policy toward higher-return behaviors. The framework thus preserves exploration while learning exploitation in the first regime, and preserves an acquired policy while learning adaptive behavior in the second. We test this formulation with QSM and QVPO on five MuJoCo tasks and with Flow-SDE+PPO on four RoboTwin manipulation tasks, finding improvements in both regimes.

Our contributions are:

\begin{itemize}[leftmargin=*]

    \item We provide a unified exploration--exploitation view of online RL and offline-to-online RL post-training: the former calls for preserving exploration while learning exploitation, whereas the latter calls for preserving an acquired policy while learning adaptive behavior;
    
    \item We introduce \textbf{CFG-Policy}, a plug-and-play framework that realizes this view by combining separate exploration and exploitation policies during generative sampling while preserving the underlying RL optimization objectives;

    \item We demonstrate the effectiveness and versatility of \textbf{CFG-Policy} across both regimes, improving QSM and QVPO on MuJoCo control tasks and Flow-SDE+PPO on RoboTwin manipulation tasks.
\end{itemize}

\section{Related Work}
\label{sec:related_work}

\begin{table}[t]
\centering
\caption{
Comparison of representative generative-policy RL methods.
CFG-Policy is a plug-and-play module applicable to compatible generative RL backbones listed above.
\cmark: explicit support;
\pmark: partial or implicit support;
\xmark: unsupported.
\textbf{E-E}: Exploration-Exploitation
}
\label{tab:related_work}
\small
\setlength{\tabcolsep}{4pt}
\renewcommand{\arraystretch}{1.12}
\resizebox{\linewidth}{!}{
\begin{tabular}{lllcccc}
\toprule
\textbf{Setting} & \textbf{Family} & \textbf{Method}
& \textbf{Exploration}
& \textbf{Adaptive E-E}
& \textbf{E-E Decoupling}
& \textbf{Capability Preserv.} \\
\midrule

\multirow{6}{*}{\textbf{Generative RL}}
& \multirow{2}{*}{Q-weighted}
& QVPO & \cmark & \pmark & \xmark & \xmark \\
& & SDAC & \cmark & \cmark & \xmark & \xmark \\

\cmidrule(lr){2-7}

& \multirow{4}{*}{Critic-Guided}
& DIPO & \pmark & \xmark & \xmark & \xmark \\
& & QSM & \pmark & \xmark & \xmark & \xmark \\
& & DACER & \cmark & \cmark & \xmark & \xmark \\
& & DIME & \cmark & \cmark & \xmark & \xmark \\

\midrule

\multirow{5}{*}{\textbf{RL Post-Training}}
& \multirow{4}{*}{\shortstack{Policy-Gradient\\Tuning}}
& DPPO & \pmark & \xmark & \xmark & \xmark \\
& & ReinFlow & \cmark & \pmark & \xmark & \xmark \\
& & Flow-SDE & \cmark & \cmark & \xmark & \xmark \\
& & Flow-Noise & \cmark & \cmark & \xmark & \xmark \\

\cmidrule(lr){2-7}

& Residual Adaptation
& Residual RL & \xmark & \xmark & \xmark & \cmark \\

\midrule

\multicolumn{3}{l}{
\textbf{Any compatible backbone + CFG-Policy (Ours)}
}
& \textbf{\cmark}
& \textbf{\cmark}
& \textbf{\cmark}
& \textbf{\cmark} \\

\bottomrule
\end{tabular}}
\end{table}

\paragraph{Classifier-Free Guidance.}
Guidance was originally introduced in diffusion models to steer generation toward desired conditions.
Classifier guidance uses gradients from an auxiliary classifier to modify the reverse diffusion process \citep{dhariwal2021diffusion}, while classifier-free guidance (CFG) instead combines conditional and unconditional model predictions during sampling \citep{ho2022classifier}.
The guidance scale controls the strength of this steering and has been widely used in text-to-image systems such as GLIDE and Imagen \citep{nichol2021glide,saharia2022photorealistic}.
Beyond conditional generation within one model, composable diffusion models combine separately modeled concepts through their generative scores \citep{liu2022composable}.
Subsequent work shows that successful score composition also depends on the sampler and studies MCMC-based corrections for compositional generation \citep{du2023reduce}.

\paragraph{Generative Reinforcement Learning.}
Offline RL established diffusion policies as expressive models of multimodal behavior: Diffusion-QL combines behavior modeling with Q-value maximization \citep{wang2023diffusionql}, while IDQL reweights actions sampled from a diffusion behavior policy using critic estimates \citep{hansenestruch2023idql}.
Online diffusion RL must improve policies through interaction rather than samples from a fixed dataset \citep{ding2026sample}.
QVPO uses a Q-weighted variational objective \citep{ding2024qvpo}, and SDAC uses reweighted score matching to fit value-induced maximum-entropy policies \citep{ma2025sdac}.
DIPO improves sampled actions with $\nabla_a Q$ before refitting the policy \citep{yang2023dipo}, whereas QSM aligns the policy score with the critic gradient \citep{psenka2024qsm}.
DACER, DIME, and MaxEntDP further study actor--critic updates and entropy-aware optimization for diffusion policies \citep{wang2024dacer,celik2025dime,dong2025maxentdp}.
ExDM explores diffusion models for unsupervised reinforcement learning by modeling state visitation distributions and deriving intrinsic rewards from distributional novelty, encouraging agents to discover rarely visited states \citep{ying2026exploratory}.
Separately, CFGRL treats CFG as a controllable improvement operator for extracting a policy from offline data \citep{frans2025cfgrl}.
These methods illustrate different ways to optimize or guide a generative policy, but generally leave exploration and exploitation coupled within the policy being improved.

\paragraph{VLA and Generative Policy RL Post-Training.}
Large vision-language-action (VLA) policies such as $\pi_0$ use flow matching to generate robot actions from visual and language inputs \citep{black2024pi0}.
RL post-training seeks to improve such pretrained behavior through reward feedback.
DPPO and ReinFlow develop policy-gradient methods for diffusion and flow policies \citep{ren2025dppo,zhang2025reinflow}, while $\pi_{\mathrm{RL}}$ applies online RL fine-tuning to flow-based VLA models \citep{chen2025pirl}.
RIPT-VLA uses sparse binary success rewards and interactive rollouts to post-train pretrained VLA models \citep{tan2025riptvla}, and Flow Policy Optimization addresses policy-ratio estimation for online fine-tuning of flow-matching VLAs \citep{lyu2025fpo}.
In the offline setting, Adaptive Reinforced Flow Matching controls the effect of advantage-weighted RL signals on the flow loss to stabilize VLA post-training \citep{zhang2025arfm}.
An alternative is residual adaptation, which freezes a base policy and learns an action-level correction \citep{johannink2019residual,silver2018residual,yuan2025policydecorator}.

However, online diffusion RL methods can lose exploratory diversity as the policy converges toward high-reward behaviors, while RL post-training needs a more stable and sample-efficient way to explore beyond a pretrained policy. The proposed CFG-Policy addresses these distinct challenges by composing separate exploration and exploitation policies throughout diffusion or flow sampling. In online RL settings, a frozen exploration branch continues to supply diverse behavior as the exploitation branch improves. In offline-to-online RL post-training, the pretrained policy remains frozen while a trainable generative branch explores and steers the composed policy toward higher returns. This design preserves the selected behavioral capability, enables adaptation throughout the generative trajectory rather than only through final-action residuals, and can be plugged into compatible RL backbones without changing their optimization objectives. Notably, CFGRL uses guidance as a policy-improvement operator learned from fixed offline data \citep{frans2025cfgrl}, whereas CFG-Policy combines two behaviorally distinct policies and optimizes one through RL interaction, as summarized in Table~\ref{tab:related_work}.

\section{Method}
\label{sec:method}

As shown in Figure \ref{fig:main}, we propose CFG-Policy which aims to address the two learning regimes discussed in the Introduction through a common principle: \emph{preserve one behavioral capability in a frozen generative policy while optimizing a complementary one}. The two policies are composed during diffusion or flow sampling, allowing the same mechanism to preserve exploration in online RL and pretrained behavior in RL post-training.

\begin{figure}
    \centering
    \includegraphics[width=0.9\linewidth]{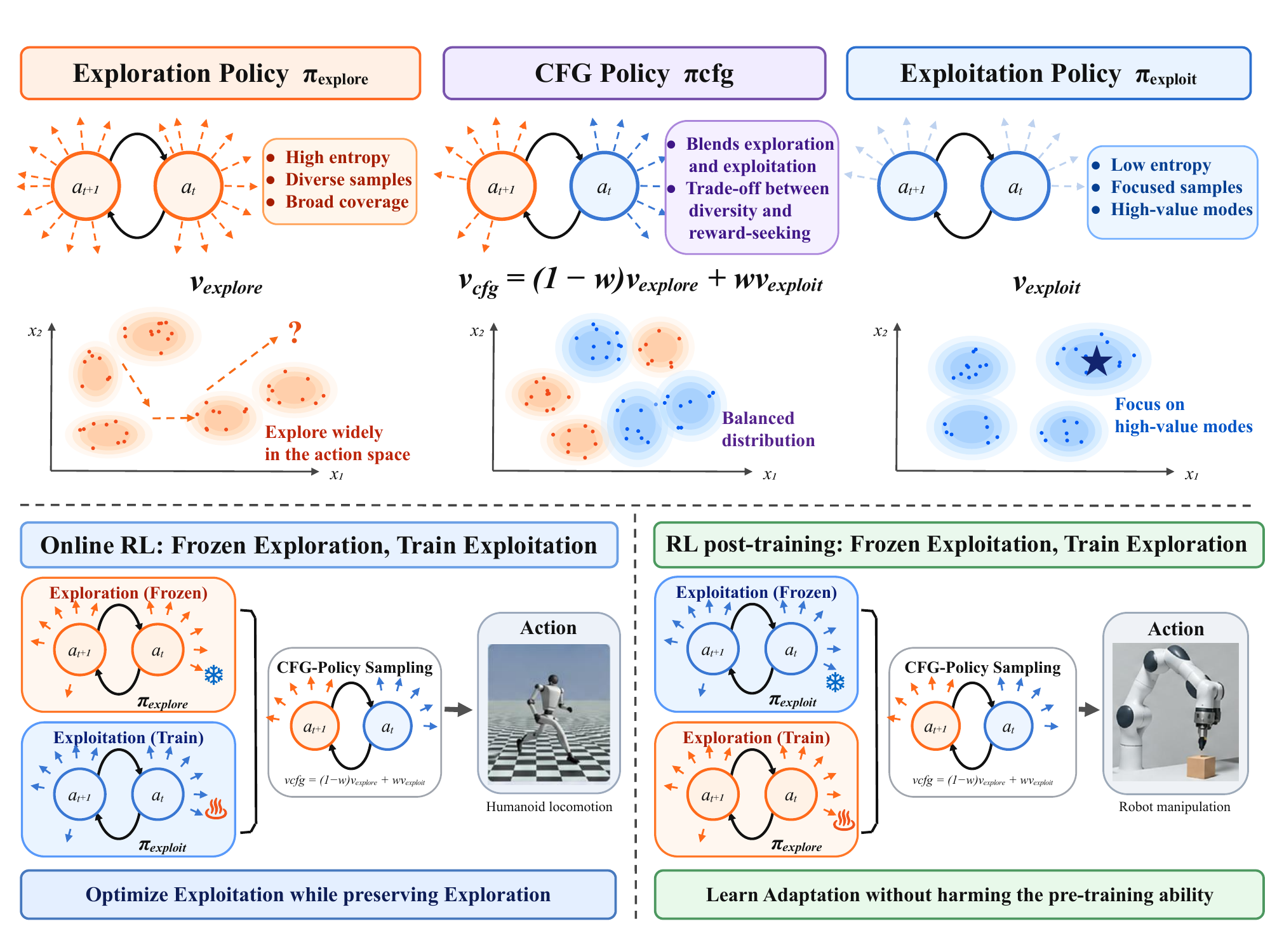}
    \caption{The sampling process of CFG-Policy. It composes two generative policies with different behavioral preferences in score space with guidance weight $w$. }
    \label{fig:main}
\end{figure}

\subsection{CFG-Policy: Compositional Generative Policies}

Consider a diffusion or flow policy that generates an action
$a=x_0$ from $x_T \sim \mathcal{N}(0,I)$. Given two generative policies
$\pi_A$ and $\pi_B$ with generative fields
$v_A(x_t,s,t)$ and $v_B(x_t,s,t)$, CFG-Policy constructs
\begin{equation}
    v_w(x_t,s,t)
    =
    (1-w)v_A(x_t,s,t)
    +
    w v_B(x_t,s,t),
    \label{eq:cfg_policy_field}
\end{equation}
where $w$ controls their relative contribution. Both fields operate on
the same intermediate sample $x_t$, so the composition steers the entire
generative trajectory rather than interpolating two independently generated
final actions.

We denote by $\pi_w^{\mathrm{sam}}(a\mid s)$ the endpoint action distribution induced by running the generative sampler with the composed field in ~\eqref{eq:cfg_policy_field}.

\begin{proposition}[Endpoint Interpretation under Path Consistency]
\label{prop:endpoint_composition}

For score-based diffusion policies, define the geometric composition at
each noise level as
\begin{equation} \label{eq:cfg_fixed_level}
    q_{w,t}(x_t\mid s)
    \propto
    \pi_{A,t}(x_t\mid s)^{1-w}
    \pi_{B,t}(x_t\mid s)^w.
\end{equation}
By ~\eqref{eq:cfg_fixed_level}, the composed score is exactly
$\nabla_{x_t}\log q_{w,t}$ at every fixed $t$.

If the family $\{q_{w,t}\}_{t=0}^T$ is approximately consistent with the
shared forward diffusion and its terminal marginal is close to the common
diffusion prior, then the endpoint distribution induced by the composed
reverse sampler satisfies
\begin{equation}
    \pi_w^{\mathrm{sam}}(a\mid s)
    \approx
    q_{w,0}(a\mid s)
    \propto
    \pi_A(a\mid s)^{1-w}
    \pi_B(a\mid s)^w.
    \label{eq:endpoint_approx}
\end{equation}
\end{proposition}


The geometric composition is exact at each fixed noise level, but need not
form an exact diffusion marginal path across time. Appendix~\ref{app:cfg_composition}
analyzes this path-consistency gap and provides the conditions underlying the
endpoint approximation.

\subsection{Online RL: Preserve Exploration, Learn Exploitation}
\label{sec:online_cfg}

In online RL, reward optimization may progressively reduce policy diversity, hence harming exploration. CFG-Policy maintains a frozen exploration policy $\pi_E$ to keep a stable ability of exploration and learns an exploitation policy $\pi_X$ from task rewards.
During environment interaction, we use

\begin{equation}
v_w=(1-w)v_E+w v_X .
\label{eq:online_cfg}
\end{equation}

The composed policy $\pi_w$ serves as the behavior policy for data collection, while the critic and $\pi_X$ are updated using the original objectives of the base RL algorithm. The exploration branch remains fixed throughout training, providing a persistent source of exploratory behavior.

The guidance coefficient directly controls this trade-off. Under a simple sharpening assumption, its effect can be characterized through policy entropy.

\begin{theorem}[Entropy of the Geometric Composition]
\label{thm:entropy}

For a fixed state $s$, define the endpoint geometric-composition surrogate
\begin{equation}
    q_w(a\mid s)
    \propto
    \pi_E(a\mid s)^{1-w}
    \pi_X(a\mid s)^w .
    \label{eq:geometric_policy}
\end{equation}
Assume that the exploitation policy is a sharpened version of the
exploration policy,
\begin{equation}
    \pi_X(a\mid s)
    =
    \frac{
        \pi_E(a\mid s)^\beta
    }{
        Z_\beta(s)
    },
    \qquad
    \beta > 1.
    \label{eq:sharpening}
\end{equation}
Then, for $w_1 < w_2$,
\begin{equation}
    H\!\left(q_{w_1}(\cdot\mid s)\right)
    \ge
    H\!\left(q_{w_2}(\cdot\mid s)\right).
    \label{eq:entropy_monotonic}
\end{equation}
\end{theorem}

\begin{algorithm}[t]
\caption{\textsc{CFG-Policy} for Online Reinforcement Learning}
\label{alg:online}
\begin{algorithmic}[1]
\Require Frozen $\pi_E$; trainable $\pi_X$; critic $Q$; replay buffer $\mathcal{B}$;
base RL algorithm $\mathcal{A}$; guidance weight $w$
\State Initialize $\pi_X,Q,\mathcal{B}$ and freeze $\pi_E$
\For{each environment step}
    \State Sample $a\sim\pi_w(\cdot|s)$ using
    $v_w\leftarrow(1-w)v_E+wv_X$
    \State Execute $a$ and store $(s,a,r,s',d)$ in $\mathcal{B}$
    \State Update $Q$ and $\pi_X$ using the original objectives of $\mathcal{A}$
\EndFor
\State \textbf{Evaluation:} set $w=1$ and execute $\pi_X$
\end{algorithmic}
\end{algorithm}

Thus, increasing $w$ makes the geometric-composition surrogate more
concentrated toward the exploitation policy. Under the path-consistency
conditions of Proposition~\ref{prop:endpoint_composition}, the sampler-induced
endpoint policy $\pi_w^{\mathrm{sam}}$ is approximated by $q_w$, providing
an idealized explanation of how guidance controls exploration and exploitation.
The algorithm for online RL is summarized in Algorithm \ref{alg:online} and the proof is provided in Appendix~\ref{app:entropy_proof}.


\subsection{RL Post-Training: Preserve Pretrained Behavior, Learn Adaptation} \label{sec:posttraining_cfg}

RL post-training presents the complementary setting: a pretrained generative policy already provides useful behavior, while direct RL updates may alter its parameters. CFG-Policy therefore freezes the pretrained policy $\pi_X^0$ and optimizes an additional steering policy $\pi_E^\theta$:

\begin{equation}
v_w^\theta
=
(1-w)v_E^\theta+w v_X^0.
\label{eq:posttraining_cfg}
\end{equation}

The composed policy is used for rollout and, when required, for likelihood recomputation. The base RL objective is kept unchanged, while gradients are stopped through $v_X^0$ and only $\pi_E^\theta$ is optimized. Unlike action-space residuals, the trainable branch steers the full generative trajectory while leaving the pretrained action generator unchanged.

The two regimes are therefore symmetric: online RL preserves exploration while learning exploitation, whereas RL post-training preserves a pretrained policy while learning adaptive behavior. The composition mechanism itself remains unchanged. The algorithm for RL post-training is summarized in Algorithm \ref{alg:posttraining}.

\begin{algorithm}[t]
\caption{\textsc{CFG-Policy} for RL Post-Training}
\label{alg:posttraining}
\begin{algorithmic}[1]
\Require Frozen pretrained $\pi_X^0$; trainable steering policy $\pi_E^\theta$;
base RL algorithm $\mathcal{A}$; guidance weight $w$
\State Freeze $\pi_X^0$ and initialize $\pi_E^\theta$
\For{each RL iteration}
    \State Construct $v_w^\theta\leftarrow(1-w)v_X^0+w v_E^\theta$
    \State Collect rollouts with the composed policy
    \State Compute the original RL objective of $\mathcal{A}$
    \State Stop gradients through $\pi_X^0$ and update only $\pi_E^\theta$
\EndFor
\State \Return $(\pi_X^0,\pi_E^\theta,w)$
\end{algorithmic}
\end{algorithm}

\subsection{Adaptive Composition} \label{sec:adaptive_cfg}

A fixed guidance weight may not provide an appropriate exploration--exploitation balance throughout training. Inspired by TRPO \citep{schulman2015trust}, we therefore adapt a global weight $w$ using a critic-value surrogate subject to a lower bound on a pathwise discrepancy from pure exploitation:
\begin{equation}
\max_{w\in(0,1)} J_Q(w)
\qquad
\mathrm{s.t.}
\qquad
D_{\mathrm{path}}(w)\geq\delta .
\label{eq:adaptive_w}
\end{equation}
The constraint is intended to discourage premature collapse toward pure exploitation. In another word, we would like the policy to output an action with larger Q value while behave differently from the purely exploiting policy. 

At each control update, the component policies and the critic are held fixed. For states drawn from the replay distribution $\rho_{\mathcal B}$, the value surrogate evaluates the actions produced by the guided sampler, including critic-based candidate selection when used:
\begin{equation}
J_Q(w)
=
\mathbb{E}_{s\sim\rho_{\mathcal B},\,\xi}
\left[
\max_{n=1,\ldots,N}
Q_{\phi}\!\left(s,a_w^{(n)}\right)
\right],
\end{equation}
where $\xi$ denotes sampling randomness and $N$ is the number of action candidates.

Using the conditional-KL decomposition of diffusion paths developed in DPOK and BDPO \citep{gao2025behavior,fan2023dpok}, we define a pathwise discrepancy between the composed and pure exploitation samplers. We estimate the trajectory expectation by Monte Carlo sampling, normalize by the action dimension, and aggregate across states using the median. Along each trajectory, we compare the composed and pure-exploitation reverse-transition means at the same intermediate state:


\begin{equation}
\widehat D_{\mathrm{path}}(w)
=
\operatorname{median}_{s\in\mathcal S}
\left[
\frac{1}{d_a RN}
\sum_{r=1}^{R}\sum_{n=1}^{N}\sum_{t=0}^{T-1}
\frac{
\left\|
\mu_t^{w}(x_t^{(r,n)},s)
-
\mu_t^{1}(x_t^{(r,n)},s)
\right\|_2^2
}{
2\sigma_t^2
}
\right].
\label{eq:dpath}
\end{equation}

Here, $\mathcal S$ is a batch of replay states and $d_a$ is the
action dimension. At stochastic steps, $\sigma_t^2$ is the common
reverse-transition variance. At the deterministic terminal step,
we use a positive auxiliary variance solely for discrepancy
evaluation without additional noise to the executed actions.
The resulting median statistic measures an auxiliary pathwise
discrepancy before candidate selection, rather than the KL
divergence between the final executed-action distributions.

The threshold $\delta$ is calibrated after a shared warmup using a prespecified reference weight and is subsequently held fixed. We use the Lagrangian $\mathcal L(w,\lambda)=J_Q(w)+\lambda\bigl(D_{\mathrm{path}}(w)-\delta\bigr), \lambda\geq0,$ with local ascent updates for $w$ and dual descent updates for$\lambda$. Finite-difference estimates at neighboring weights reuse the same states and sampling noise. These updates optimize a critic surrogate and do not guarantee monotonic improvement in realized return. Details of the discrepancy estimator, calibration, and scalar updates are provided in Appendix~\ref{app:adaptive}.

\section{Experiments}\label{exp}


We evaluate CFG-Policy in the two complementary regimes introduced in the Method: online RL and RL post-training. Our experiments study whether
(i) CFG-Policy can be applied to different generative RL algorithms without changing their original objectives,
(ii) the same compositional mechanism is effective in both learning regimes, and
(iii) adaptive generative composition is preferable to fixed guidance or simple action-space blending.

\subsection{Experimental Setup}

\paragraph{Online RL.} We evaluate on five MuJoCo tasks: HalfCheetah-v4, Walker2d-v4, Ant-v4, Humanoid-v4, and Hopper-v4. CFG-Policy is instantiated with QSM~\citep{psenka2024qsm} and QVPO~\citep{ding2024qvpo}. The original actor and critic objectives are unchanged. The exploitation branch is inherited from the base algorithm, while a separately pretrained RND exploration policy~\citep{burda2018exploration}, which we use it for the exploration branch, is kept frozen during RL training. We report mean performance and sample standard deviation over 5 runs.

\paragraph{RL post-training.} We evaluate on four RoboTwin tasks~\citep{chen2025robotwin}: Adjust Bottle, Pick Dual Bottles, Rotate QR Code, and Grab Roller. For both $\pi_0$ and $\pi_{0.5}$ \citep{black2024pi0,intelligence2025pi_}, the pretrained policy is frozen as the exploitation branch, while a second branch initialized from the same SFT checkpoint is optimized with RL. We compare against direct Flow-SDE+PPO fine-tuning and report task success rates. 

We provide introduction of the tasks of benchmark and further implementation details in Appendix \ref{app:environmental_details} and \ref{app:impl_details}.

\begin{table*}[t]
\centering
\caption{Results across five environments.
Values are mean $\pm$ sample standard deviation over 5 runs. The maximum value for each task is shown in bold.}
\label{tab:online_results}
\resizebox{\textwidth}{!}{%
\begin{tabular}{lccccc}
\toprule
Method & HalfCheetah-v4 & Walker2d-v4 & Ant-v4 & Humanoid-v4 & Hopper-v4 \\
\midrule
QSM Original
& $5020 \pm 455$ & $2041 \pm 477$ & $-24 \pm 3$
& $5922 \pm 262$ & $2753 \pm 307$ \\
QSM CFG-Policy
& $6218 \pm 533$ & $2349 \pm 583$ & $1425 \pm 308$
& $6094 \pm 169$ & $3101 \pm 461$ \\
Improvement
& $+23.9\%$ & $+15.1\%$ & $+6141.3\%$
& $+2.90\%$ & $+12.7\%$ \\
\midrule
QVPO Original
& $11166 \pm 5076$ & $5220 \pm 689$ & $7129 \pm 61$
& $6612 \pm 214$ & $3727 \pm 176$ \\
QVPO CFG-Policy
& $\mathbf{12990 \pm 1993}$ & $\mathbf{5692 \pm 413}$
& $\mathbf{7374 \pm 45}$ & $\mathbf{7041 \pm 385}$
& $\mathbf{3871 \pm 192}$ \\
Improvement
& $+16.3\%$ & $+9.0\%$ & $+3.4\%$
& $+6.49\%$ & $+3.8\%$ \\
\bottomrule
\end{tabular}%
}
\end{table*}

\begin{figure}
    \centering
    \includegraphics[width=1\linewidth]{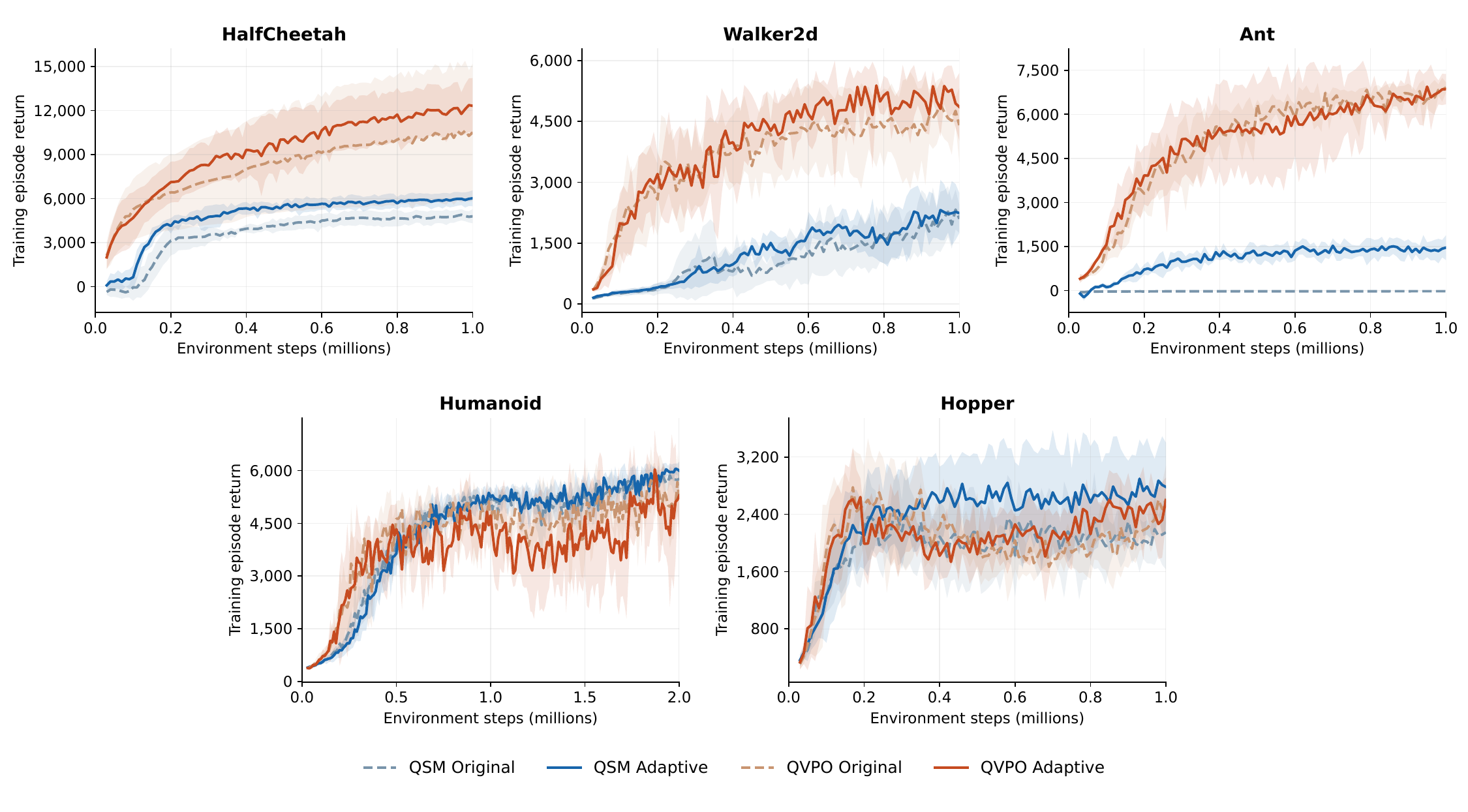}
    \caption{Learning curve of QSM, QVPO before and after using CFG-Policy on five MuJoCo tasks. The x-axis is the environment steps and the y-axis shows average episodic return with one standard deviation shaded.}
    \label{fig:online_curves}
\end{figure}





\subsection{Online Reinforcement Learning}

As shown in Table \ref{tab:online_results}, CFG-Policy consistently improves both QSM and QVPO on all evaluated tasks. For QSM, the improvement is particularly significant on Ant-v4, where CFG-Policy increases the return from $-24$ to $1425 \pm 308$. It also improves HalfCheetah-v4, Walker2d-v4, and Hopper-v4, demonstrating that preserving a separate exploration branch can provide additional exploration capability during reward optimization.

The benefit is not limited to weaker baselines. Starting from the stronger QVPO baseline, CFG-Policy further improves performance across all five environments, including HalfCheetah-v4 and Humanoid-v4. Figure \ref{fig:online_curves} shows consistent improvements throughout training, demonstrating that compositional sampling enhances existing RL objectives without modification. Despite introducing an additional generative branch, CFG-Policy incurs only marginal computational overhead in memory and training time. We provide the cost details in Appendix \ref{app:cost}.






\begin{table}[t]
\centering
\caption{Success rates of the original Flow-SDE + PPO and CFG-Policy.}
\label{tab:posttraining_results}
\begin{tabular}{llcccc}
\toprule
Model & Method
& Adjust Bottle
& Pick Dual Bottles
& Rotate QR Code
& Grab Roller \\
\midrule
\multirow{2}{*}{$\pi_0$}
& FlowSDE+PPO
& 82.81\% & 43.75\% & 34.38\% & 92.19\% \\
& CFG-Policy
& \textbf{98.44\%}
& \textbf{57.81\%}
& \textbf{42.19\%}
& \textbf{93.75\%} \\
\midrule
\multirow{2}{*}{$\pi_{0.5}$}
& FlowSDE+PPO
& 89.06\% & 12.50\% & 7.81\% & \textbf{100.00\%} \\
& CFG-Policy
& \textbf{92.19\% }
& \textbf{15.62\%}
& \textbf{10.94\% }
& \textbf{100.00\%} \\
\bottomrule
\end{tabular}
\end{table}

\subsection{RL Post-Training}

We next evaluate CFG-Policy for RL post-training with pretrained generative policies. Unlike online RL, where the exploration branch is preserved, here the pretrained exploitation policy is frozen and a separate steering policy is optimized.

Table \ref{tab:posttraining_results} compares CFG-Policy with direct Flow-SDE+PPO fine-tuning on four RoboTwin manipulation tasks using both $\pi_0$ and $\pi_{0.5}$. CFG-Policy improves non-saturated tasks while maintaining saturated performance.

For $\pi_0$, CFG-Policy increases the success rate on Adjust Bottle from $82.81\%$ to $98.44\%$, and on Pick Dual Bottles from $43.75\%$ to $57.81\%$. Similar improvements are observed for $\pi_{0.5}$ on the non-saturated tasks, while both methods reach full success on Grab Roller.

These results show that generative composition enables policy adaptation through an additional trainable branch while preserving the original pretrained action generator.




\begin{figure}
    \centering
    \includegraphics[width=0.9\linewidth]{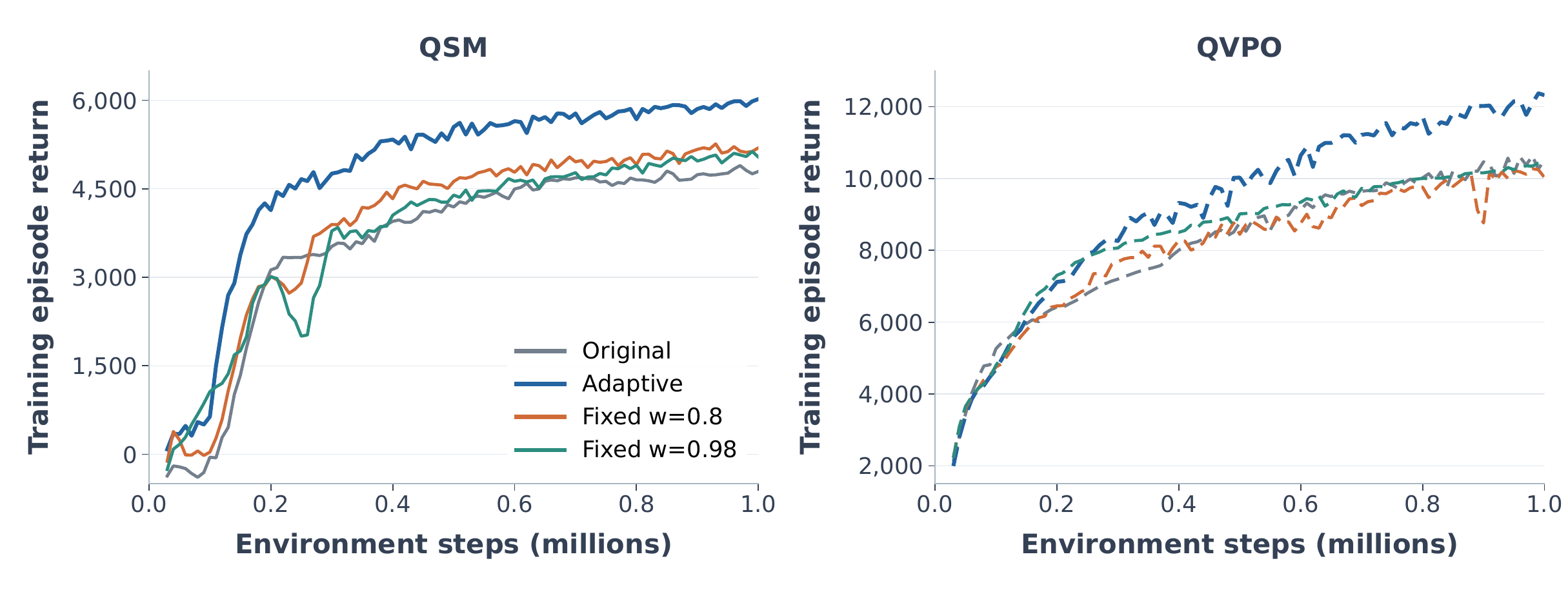}
    \caption{Comparison of adaptive guidance with two fixed guidance coefficients. The dashed line corresponds to QVPO and the solid line corresponds to QSM.}
    \label{fig:ablation}
\end{figure}

\subsection{Ablation Study}

We investigate two design choices of CFG-Policy:
(1) adaptive guidance coefficient selection and 
(2) generative-space composition compared with action-space interpolation.

\textbf{Adaptive versus fixed guidance.}
Figure 4 compares adaptive guidance with fixed coefficients $w\in\{0.8,0.98\}$ on HalfCheetah-v4.
For both QSM and QVPO, adaptive guidance achieves higher final training returns than the tested fixed coefficients.
This indicates that dynamically adjusting the exploration--exploitation trade-off can provide a more flexible alternative than using a single manually selected guidance coefficient throughout training.

\textbf{Generative composition versus action-space blending.}
We further compare CFG-Policy with directly interpolating actions sampled from the exploration and exploitation policies.
As shown in Table 4, action-space blending consistently underperforms the original QVPO policy, even when using a large exploitation coefficient.

In contrast, CFG-Policy combines generative fields during the sampling trajectory, allowing both policies to influence the denoising process. The results highlight the advantage of composing policies in the generative space rather than interpolating independently generated actions.

\begin{table}[t]
\centering
\small
\caption{Action-space blending ablation on HalfCheetah-v4. Results are the mean and sample standard deviation across five seeds, with ten evaluation episodes per seed.}
\label{tab:halfcheetah_action_blend}
\begin{tabular}{ccc}
\toprule
Blend coefficient $a$ & Return & Change vs.\ original QVPO \\
\midrule
$0$    & $-599 \pm 171$                   & $-105.44\%$ \\
$0.2$  & $-362 \pm 127$                   & $-103.29\%$ \\
$0.5$  & $883 \pm 608$                    & $-91.98\%$ \\
$0.8$  & $4078 \pm 1132$            & $-62.96\%$ \\
$0.98$ & $10642 \pm 4787$           & $-3.33\%$ \\
$1.0$  & $11166 \pm 5076$ & -- \\
\bottomrule
\end{tabular}
\label{tab:ablation}
\end{table}

\section{Conclusion and Limitation}
\label{sec:conclusion}
In this paper, we introduced \textbf{CFG-Policy} to address exploration--exploitation challenges across online diffusion RL and offline-to-online RL post-training. By composing separate generative policies during diffusion or flow sampling, it preserves one behavioral capability while the other is optimized using the original RL objective. In online RL, a frozen exploration branch sustains behavioral diversity as exploitation improves; in post-training, a frozen pretrained policy retains its acquired capabilities while a trainable branch steers behavior. With QSM and QVPO on 5 MuJoCo tasks, CFG-Policy achieves higher final returns and faster convergence, improving sample efficiency. With Flow-SDE+PPO on 4 RoboTwin manipulation tasks, it improves or matches task success rates. These results support CFG-Policy as a plug-and-play sampling interface for both learning regimes.
Nevertheless, the preferred guidance weight \(w\) varies substantially across tasks. Although our adaptive strategy reduces the need for manual selection, task-specific tuning can still find weights that outperform it. Developing more effective task-aware guidance adaptation is therefore an important direction for future work.

\clearpage



\section*{Acknowledgment}
This work was supported by the National Natural Science Foundation of China (62303319, 62406195), HPC Platform of ShanghaiTech University, and MoE Key Laboratory of Intelligent Perception and Human-Machine Collaboration (ShanghaiTech University), Shanghai Engineering Research Center of Intelligent Vision and Imaging. This work was also supported in part by computational resources provided by Fcloud CO., LTD.

\bibliography{iclr2027_conference}
\bibliographystyle{iclr2027_conference}

\appendix


\newpage
\appendix

\section{Environmental Details}
\label{app:environmental_details}

We evaluate on nine continuous-control tasks spanning two benchmarks:
five MuJoCo locomotion environments and four RoboTwin manipulation
environments (Figure~\ref{fig:benchmarks}). The two domains probe
complementary aspects of policy learning: online interaction in
locomotion and reinforcement-learning post-training in object
manipulation.

\begin{figure}[ht]
    \centering
    \includegraphics[width=\linewidth]{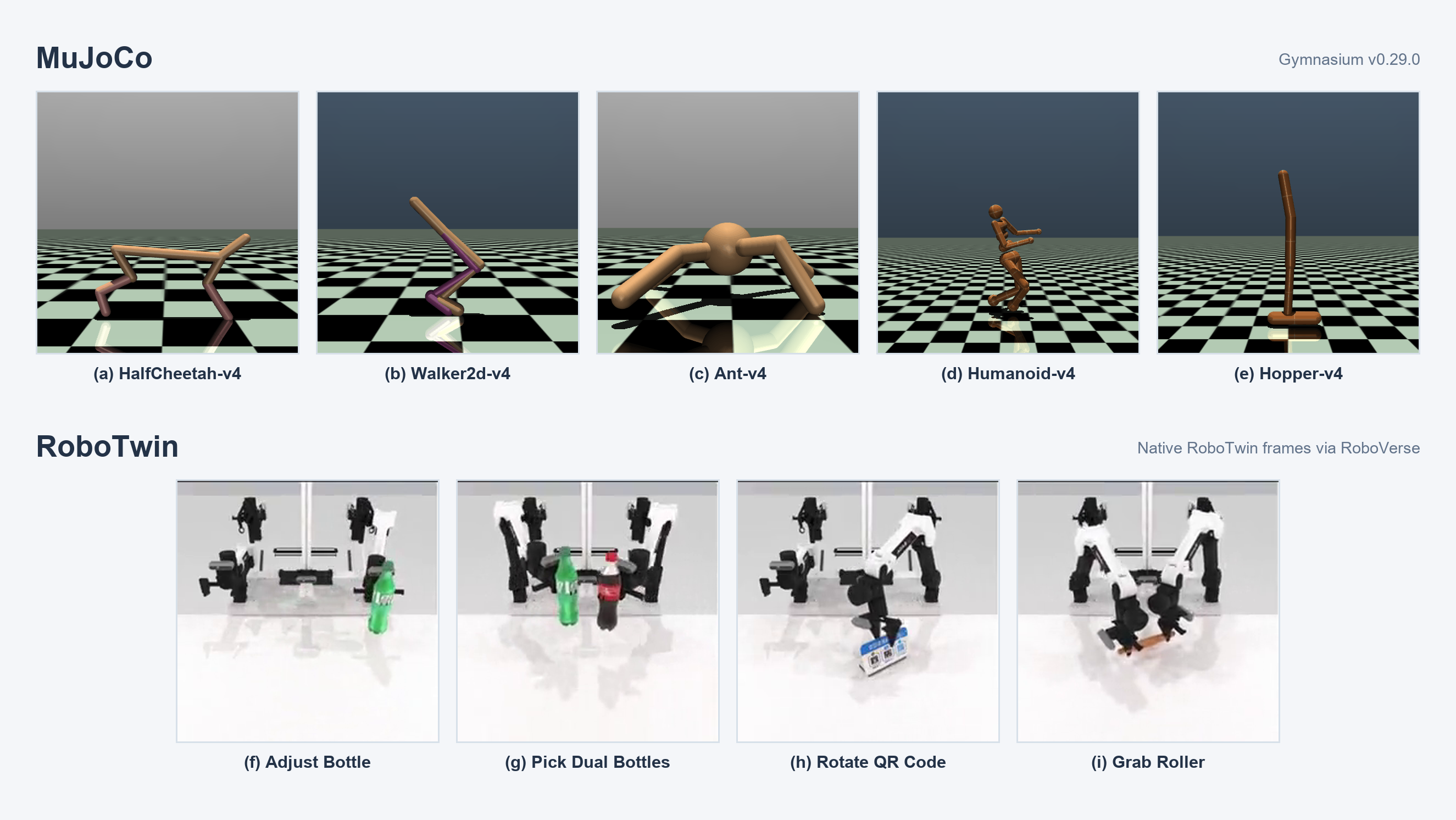}
    \caption{Evaluation environments. (a) HalfCheetah-v4,
    (b) Walker2d-v4, (c) Ant-v4, (d) Humanoid-v4, and
    (e) Hopper-v4 are MuJoCo locomotion tasks. (f) Adjust Bottle,
    (g) Pick Dual Bottles, (h) Rotate QR Code, and (i) Grab Roller
    are RoboTwin manipulation tasks.}
    \label{fig:benchmarks}
\end{figure}

\paragraph{MuJoCo locomotion.}
We use the \texttt{v4} versions of five Gymnasium MuJoCo tasks. Each agent receives a continuous observation and
controls its joints through continuous actions. The tasks reward
forward movement while imposing task-dependent control and stability
requirements:
\begin{itemize}
    \item \textbf{HalfCheetah-v4} ($d_o=17$, $d_a=6$): coordinate the
    front and rear legs of a planar cheetah-like robot to run forward
    rapidly while limiting control effort.
    \item \textbf{Walker2d-v4} ($d_o=17$, $d_a=6$): generate a stable
    alternating gait for a planar biped and maintain an upright torso.
    \item \textbf{Ant-v4} ($d_o=27$, $d_a=8$): coordinate four articulated
    legs to move a three-dimensional quadruped forward while keeping
    its body stable.
    \item \textbf{Humanoid-v4} ($d_o=376$, $d_a=17$): control a
    three-dimensional humanoid to move forward without falling,
    requiring coordinated whole-body motion.
    \item \textbf{Hopper-v4} ($d_o=11$, $d_a=3$): propel a planar
    single-legged robot forward through repeated balanced hops.
\end{itemize}
Here $d_o$ and $d_a$ denote the default observation and action
dimensions, respectively.

\paragraph{RoboTwin manipulation.}
We select four RoboTwin tasks that cover single-arm object
reorientation and coordinated dual-arm grasping \cite{chen2025robotwin}:
\begin{itemize}
    \item \textbf{Adjust Bottle}: select the appropriate arm and pick
    up a bottle with its cap facing upward.
    \item \textbf{Pick Dual Bottles}: pick up two bottles, using one
    arm for each bottle.
    \item \textbf{Rotate QR Code}: grasp a QR-code payment sign and
    rotate it until the code faces the robot.
    \item \textbf{Grab Roller}: use both arms to grasp the same roller
    on the tabletop.
\end{itemize}

\paragraph{Evaluation protocol.}
For online reinforcement learning in MuJoCo, we report the mean
episodic return and sample standard deviation over five independent
runs. For RoboTwin post-training, we evaluate the optimized policy
using the task success rate. These metrics are reported separately
because episodic return and manipulation success measure different
outcomes.

\subsection{Algorithm Details} \label{app:algo_details}

\paragraph{CFG-Policy for Online Reinforcement Learning.}
We consider a fixed exploration policy $\pi_E$, a trainable exploitation policy $\pi_X$, a critic $Q$, a replay buffer $\mathcal{B}$, an underlying RL optimizer $\mathcal{A}$, and a guidance coefficient $w$. We first initialize $\pi_X$, $Q$, and $\mathcal{B}$, while keeping the parameters of $\pi_E$ frozen throughout training. At each environment step, both policies evaluate the same intermediate generative state conditioned on the current observation $s_t$, and their generative fields are composed as
\[
v_w=(1-w)v_E+w v_X.
\]
The original sampler is otherwise left unchanged, and the resulting composed policy $\pi_w$ is used to generate the behavior action $a_t$. After executing $a_t$, the transition $(s_t,a_t,r_t,s_{t+1})$ is stored in $\mathcal{B}$. The optimizer $\mathcal{A}$ then updates the critic and the exploitation policy using its original RL objective, with gradients blocked through $\pi_E$. Thus, CFG-Policy only modifies the sampling interface while preserving the optimization rule of the underlying RL algorithm. After training, the learned exploitation policy $\pi_X$ is returned.

\begin{algorithm}[ht]
\caption{CFG-Policy for Online Reinforcement Learning}
\label{alg:online_appendix}
\begin{algorithmic}[1]
\Require Frozen exploration policy \(\pi_E\), trainable exploitation policy
\(\pi_X\), critic \(Q\), replay buffer \(\mathcal{B}\), RL optimizer \(\mathcal{A}\),
guidance coefficient \(w\).
\State Initialize \(\pi_X\), \(Q\), and \(\mathcal{B}\).
\State Freeze parameters of \(\pi_E\).
\For{each environment interaction step}
    \State Observe state \(s_t\).
    \State Construct the composed generative field
    \(v_w=(1-w)v_E+w v_X\).
    \State Sample action \(a_t\sim\pi_w(\cdot|s_t)\) using the original sampler.
    \State Execute \(a_t\) and store
    \((s_t,a_t,r_t,s_{t+1})\) in \(\mathcal{B}\).
    \State Update \(Q\) and \(\pi_X\) using the original RL objective.
\EndFor
\State \textbf{return} exploitation policy \(\pi_X\).
\end{algorithmic}
\end{algorithm}

\paragraph{CFG-Policy for RL Post-Training.}
We consider a frozen pretrained exploitation policy $\pi_X^0$, a trainable steering policy $\pi_E^\theta$, an underlying RL optimizer $\mathcal{A}$, and a guidance coefficient $w$. The base policy $\pi_X^0$ is loaded from the pretrained checkpoint and kept fixed throughout RL post-training, while $\pi_E^\theta$ is initialized from the pretrained action-generation module so that the two branches start from compatible generative dynamics. At each RL iteration, both policies evaluate the same intermediate generative state under the same observation context, and their fields are composed as
\[
v_w=(1-w)v_X^0+w v_E^\theta.
\]
The resulting composed policy $\pi_w$ is then used to generate trajectories with the original generative sampler. The optimizer $\mathcal{A}$ computes its original RL objective from the collected trajectories and updates only the steering policy, with gradients blocked through the frozen base policy. In our implementation, we instantiate this procedure with Flow-SDE sampling and PPO with generalized advantage estimation, while the compositional mechanism itself is independent of these choices. After training, inference uses the frozen base policy together with the learned steering policy and the guidance coefficient $w$.

\begin{algorithm}[t]
\caption{CFG-Policy for RL Post-Training}
\label{alg:post_appendix}
\begin{algorithmic}[1]
\Require Frozen pretrained policy \(\pi_X^0\), trainable steering policy
\(\pi_E^\theta\), RL optimizer \(\mathcal{A}\), guidance coefficient \(w\).
\State Freeze parameters of \(\pi_X^0\).
\State Initialize \(\pi_E^\theta\) from the pretrained action-generation module.
\For{each RL iteration}
    \State Construct the composed generative field
    \(v_w=(1-w)v_X^0+w v_E^\theta\).
    \State Generate trajectories using the composed policy \(\pi_w\).
    \State Compute the original RL objective \(\mathcal{L}_{\mathrm{RL}}(\pi_w)\).
    \State Stop gradients through \(\pi_X^0\).
    \State Update only \(\pi_E^\theta\).
\EndFor
\State \textbf{return} composed policy \((\pi_X^0,\pi_E^\theta,w)\).
\end{algorithmic}
\end{algorithm}
\section{Theoretical Analysis of CFG Policy Composition}

\subsection{CFG Composition as Distribution Multiplication and Path Consistency}
\label{app:cfg_composition}

We provide the distributional interpretation of CFG composition for
score-based diffusion policies. The key distinction is that geometric
composition is exact at every fixed noise level, while the corresponding
family of densities is not necessarily an exact diffusion marginal path.

\paragraph{Fixed-noise-level composition.}

For a fixed state $s$, let
$p_{A,t}(x_t\mid s)$ and $p_{B,t}(x_t\mid s)$ denote the noisy marginals
of two diffusion policies at noise level $t$. Their scores are
\begin{equation}
    s_{A,t}(x_t,s)
    =
    \nabla_{x_t}\log p_{A,t}(x_t\mid s),
    \qquad
    s_{B,t}(x_t,s)
    =
    \nabla_{x_t}\log p_{B,t}(x_t\mid s).
\end{equation}

CFG-Policy composes the two scores as
\begin{equation}
    s_{w,t}(x_t,s)
    =
    (1-w)s_{A,t}(x_t,s)
    +
    w s_{B,t}(x_t,s).
    \label{eq:app_cfg_score}
\end{equation}

Define the geometric composition
\begin{equation}
    q_{w,t}(x_t\mid s)
    =
    \frac{
        p_{A,t}(x_t\mid s)^{1-w}
        p_{B,t}(x_t\mid s)^w
    }{
        Z_{w,t}(s)
    },
    \label{eq:app_qwt}
\end{equation}
where
\begin{equation}
    Z_{w,t}(s)
    =
    \int
    p_{A,t}(x_t\mid s)^{1-w}
    p_{B,t}(x_t\mid s)^w
    \, dx_t .
\end{equation}

Since the normalization term is independent of $x_t$,
\begin{align}
    \nabla_{x_t}\log q_{w,t}(x_t\mid s)
    &=
    (1-w)\nabla_{x_t}\log p_{A,t}(x_t\mid s)
    +
    w\nabla_{x_t}\log p_{B,t}(x_t\mid s)
    \\
    &=
    s_{w,t}(x_t,s).
    \label{eq:app_fixed_level_identity}
\end{align}

Therefore, the CFG-composed score is exactly the score of the geometric
composition $q_{w,t}$ at every fixed noise level.

\paragraph{Why fixed-level composition does not imply an exact endpoint product?}

The above identity does not imply that the family
$\{q_{w,t}\}_{t=0}^T$ is itself generated by applying the original forward
diffusion to the endpoint density $q_{w,0}$.

For the VP diffusion used in our score-based implementations, the forward
process can be written as
\begin{equation}
    x_t
    =
    \alpha_t x_0
    +
    \sigma_t \epsilon,
    \qquad
    \epsilon\sim\mathcal N(0,I).
\end{equation}
Let $\mathcal T_t$ denote the corresponding forward noising operator,
\begin{equation}
    \mathcal T_t[p](x_t)
    =
    \int
    \mathcal N(x_t;\alpha_t x_0,\sigma_t^2 I)
    p(x_0)\,dx_0 .
    \label{eq:vp_operator}
\end{equation}

Then
\begin{equation}
    p_{A,t}
    =
    \mathcal T_t[p_A],
    \qquad
    p_{B,t}
    =
    \mathcal T_t[p_B],
\end{equation}
and hence
\begin{equation}
    q_{w,t}
    \propto
    \left(\mathcal T_t[p_A]\right)^{1-w}
    \left(\mathcal T_t[p_B]\right)^w.
    \label{eq:qwt_forward}
\end{equation}

In contrast, directly diffusing the endpoint geometric composition gives
\begin{equation}
    \bar q_{w,t}
    =
    \mathcal T_t[q_{w,0}]
    \propto
    \mathcal T_t
    \left[
        p_A^{1-w}p_B^w
    \right].
    \label{eq:qbar}
\end{equation}

In general,
\begin{equation}
    \mathcal T_t
    \left[
        p_A^{1-w}p_B^w
    \right]
    \not\propto
    \left(\mathcal T_t[p_A]\right)^{1-w}
    \left(\mathcal T_t[p_B]\right)^w.
    \label{eq:non_commute}
\end{equation}
Thus, forward diffusion and geometric composition do not generally commute.
Consequently, the exact fixed-noise-level identity in
~\eqref{eq:app_fixed_level_identity} does not by itself imply an exact
product form for the endpoint distribution generated by the composed reverse
sampler.

\paragraph{Path-consistency residual.}

We next characterize the discrepancy locally in diffusion time. Consider
the shared forward SDE
\begin{equation}
    dX_t
    =
    f_t(X_t)\,dt
    +
    g_t\,dW_t,
\end{equation}
with Fokker--Planck operator
\begin{equation}
    \mathcal L_t^\ast p
    =
    -\nabla_x\cdot(f_t p)
    +
    \frac{g_t^2}{2}\Delta_x p.
\end{equation}
Both component marginals satisfy
\begin{equation}
    \partial_t p_{A,t}
    =
    \mathcal L_t^\ast p_{A,t},
    \qquad
    \partial_t p_{B,t}
    =
    \mathcal L_t^\ast p_{B,t}.
\end{equation}

For the geometric composition in ~\eqref{eq:app_qwt}, direct
differentiation gives
\begin{equation}
    \partial_t q_{w,t}
    =
    \mathcal L_t^\ast q_{w,t}
    +
    q_{w,t}
    \left(
        R_{w,t}
        -
        \mathbb E_{q_{w,t}}[R_{w,t}]
    \right),
    \label{eq:app_path_residual}
\end{equation}
where
\begin{equation}
    R_{w,t}(x_t,s)
    =
    \frac{g_t^2}{2}
    w(1-w)
    \left\|
        s_{B,t}(x_t,s)
        -
        s_{A,t}(x_t,s)
    \right\|_2^2 .
    \label{eq:app_residual}
\end{equation}

Equation~\eqref{eq:app_path_residual} explicitly quantifies the
path-consistency gap. In particular, only the state-dependent variation of
$R_{w,t}$ affects the diffusion path: if $R_{w,t}$ is constant in $x_t$,
the residual is canceled by the normalization term
$\mathbb E_{q_{w,t}}[R_{w,t}]$.

We therefore define the accumulated path-consistency error
\begin{equation}
    \epsilon_{\mathrm{pc}}(s)
    =
    \int_0^T
    \mathbb E_{x_t\sim q_{w,t}(\cdot\mid s)}
    \left[
        \left|
            R_{w,t}(x_t,s)
            -
            \mathbb E_{q_{w,t}}[R_{w,t}]
        \right|
    \right]
    dt .
    \label{eq:app_pc_error}
\end{equation}

When $\epsilon_{\mathrm{pc}}$ is small, the geometric family
$\{q_{w,t}\}$ approximately follows the same diffusion dynamics as the two
component policies.

\paragraph{Endpoint interpretation.}

At sufficiently large noise levels, the two component marginals approach the
same diffusion prior, so that $q_{w,T}$ is also close to the sampler
initialization distribution. Under this terminal-matching condition and
small path-consistency error, the endpoint distribution
$\pi_w^{\mathrm{sam}}$ induced by the composed reverse sampler is therefore
well approximated by
\begin{equation}
    \pi_w^{\mathrm{sam}}(a\mid s)
    \approx
    q_{w,0}(a\mid s)
    \propto
    \pi_A(a\mid s)^{1-w}
    \pi_B(a\mid s)^w.
    \label{eq:app_endpoint_approx}
\end{equation}

Hence, distribution multiplication should be interpreted as an exact
fixed-noise-level property of CFG and an approximate endpoint interpretation
of the complete diffusion sampling process.















\subsection{Adaptive Composition}
\label{app:adaptive}

\paragraph{Pathwise divergence estimator.}
The adaptive guidance objective in Section~\ref{sec:adaptive_cfg} constrains the composed sampler to maintain sufficient deviation from the pure exploitation branch. We measure this deviation along trajectories generated by the composed policy.

For a state $s$, let $x_t^{(r)}$ denote the intermediate sample at denoising step $t$ along the $r$-th trajectory generated using guidance weight $w$. At the same intermediate sample $x_t^{(r)}$, we evaluate the reverse-transition means of the composed sampler and the pure exploitation sampler, denoted by $\mu_t^{w}(x_t^{(r)},s)$ and $\mu_t^{1}(x_t^{(r)},s)$, respectively. When the two reverse transitions share variance $\sigma_t^2 I$, their conditional Gaussian KL is

\begin{equation}
D_t^{(r)}(w;s)
=
\frac{
\left\|
\mu_t^{w}(x_t^{(r)},s)
-
\mu_t^{1}(x_t^{(r)},s)
\right\|_2^2
}{
2\sigma_t^2
}.
\label{eq:step_path_kl}
\end{equation}

In practice, we estimate the pathwise discrepancy as

\begin{equation}
\widehat D_{\mathrm{path}}(w)
=
\operatorname{median}_{s\in\mathcal S}
\left[
\frac{1}{d_a R}
\sum_{r=1}^{R}
\sum_{t=0}^{T-1}
\frac{
\left\|
\mu_t^{w}(x_t^{(r)},s)
-
\mu_t^{1}(x_t^{(r)},s)
\right\|_2^2
}{
2\sigma_t^2
}
\right],
\label{eq:path_divergence}
\end{equation}

where $\mathcal S$ is a set of states sampled from the replay buffer, $d_a$ is the action dimension, and $R$ is the number of sampled generative trajectories per state. For QVPO, when multiple candidate actions are sampled for each state, the corresponding estimates are additionally averaged over the candidates.

Importantly, $x_t^{(r)}$ is sampled from the trajectory generated by the
composed policy with guidance weight $w$. Both $\mu_t^{w}$ and $\mu_t^{1}$ are evaluated at this same $x_t^{(r)}$; we do not independently generate trajectories from the two policies and compare their resulting samples. The reference $\mu_t^{1}$ corresponds to the $w=1$ pure exploitation branch. For QSM, it retains the exploitation-strength scaling used by the original sampler.

The final reverse step is deterministic in the actual sampling procedure. For the sole purpose of evaluating  Equation~\eqref{eq:path_divergence}, we assign this step an auxiliary comparison variance. This variance is used only in the discrepancy estimator and does not introduce additional noise into the generated action.

Therefore, $\widehat D_{\mathrm{path}}$ should be interpreted as an auxiliary KL-based discrepancy accumulated along the generative trajectory, rather than as the exact endpoint divergence $D_{\mathrm{KL}}(\pi_w\|\pi_1)$.

\paragraph{Adaptive guidance optimization.}
Using the above discrepancy, we optimize the guidance weight through

\begin{equation}
\max_w J_Q(w)
\qquad
\mathrm{s.t.}
\qquad
D_{\mathrm{path}}(w)\geq \delta ,
\label{eq:adaptive_constraint_appendix}
\end{equation}

where $J_Q(w)$ is the estimated value of actions generated under guidance
weight $w$.
The corresponding Lagrangian is

\begin{equation}
\mathcal{L}(w,\lambda)
=
J_Q(w)
+
\lambda
\left(
D_{\mathrm{path}}(w)-\delta
\right).
\label{eq:adaptive_lagrangian_appendix}
\end{equation}

We update $w$ using finite-difference estimates and update the dual variable
$\lambda$ according to the constraint violation.
In implementation,
$D_{\mathrm{path}}(w)$ is replaced by the empirical estimator
$\widehat D_{\mathrm{path}}(w)$ in
Equation~\eqref{eq:path_divergence}.

\paragraph{Non-degradation under exact optimization.}
We finally justify the observation stated in Section~\ref{sec:adaptive_cfg}.
Define the feasible set

\begin{equation}
\mathcal W_\delta
=
\left\{
w:
D_{\mathrm{path}}(w)\geq\delta
\right\}.
\end{equation}

Let $w_0$ be the original guidance coefficient and assume
$w_0\in\mathcal W_\delta$.
Let

\begin{equation}
w^\star
=
\arg\max_{w\in\mathcal W_\delta}
J_Q(w).
\end{equation}

If the constrained optimization is solved globally, then by optimality,

\begin{equation}
J_Q(w^\star)
\geq
J_Q(w_0).
\end{equation}

Under the additional assumption that the critic objective equals the true return of the corresponding composed policy,

\begin{equation}
J_Q(w)=J(\pi_w),
\end{equation}

we obtain

\begin{equation}
J(\pi_{w^\star})
\geq
J(\pi_{w_0}).
\end{equation}

Thus, under these idealized assumptions, the adaptive solution cannot perform worse than the original feasible guidance coefficient.

\paragraph{Calibration of $\delta$.}
When CFG is enabled, we sample a set of states $\mathcal S_{\mathrm{cal}}$ from the replay buffer and evaluate the normalized trajectory divergence at a reference guidance coefficient $w_{\mathrm{ref}}$. We set
\begin{equation}
\delta=\kappa\,\widehat D_{\mathrm{path}}(w_{\mathrm{ref}};\mathcal S_{\mathrm{cal}}),
\end{equation}
where $\widehat D_{\mathrm{path}}$ averages over sampled trajectories for each state and takes the median across states. We use $\kappa=1$ for QVPO; QSM optionally adjusts the target through $\kappa$. The resulting $\delta$ is fixed after calibration and is not recomputed as the policies evolve.

\subsection{Proof of Theorem~\ref{thm:entropy}}
\label{app:entropy_proof}
\begin{proof}
For a fixed state $s$, assume
\begin{equation}
    \pi_X(a\mid s)
    =
    \frac{
        \pi_E(a\mid s)^\beta
    }{
        Z_\beta(s)
    },
    \qquad
    \beta>1.
\end{equation}
Substituting this relationship into the geometric-composition surrogate,
\begin{align}
    q_w(a\mid s)
    &\propto
    \pi_E(a\mid s)^{1-w}
    \pi_X(a\mid s)^w
    \\
    &\propto
    \pi_E(a\mid s)^{1-w}
    \left(
        \pi_E(a\mid s)^\beta
    \right)^w
    \\
    &=
    \pi_E(a\mid s)^{1+w(\beta-1)} .
\end{align}

Define
\begin{equation}
    \alpha(w)
    =
    1+w(\beta-1).
\end{equation}
Then
\begin{equation}
    q_w(a\mid s)
    =
    \frac{
        \pi_E(a\mid s)^{\alpha(w)}
    }{
        Z(\alpha(w),s)
    },
\end{equation}
where
\begin{equation}
    Z(\alpha,s)
    =
    \int
    \pi_E(a\mid s)^\alpha
    \, da .
\end{equation}

For brevity, we omit the conditioning on $s$ below. The entropy is
\begin{align}
    H(q_w)
    &=
    -\mathbb E_{q_w}
    \left[
        \log q_w(a)
    \right]
    \\
    &=
    -\alpha
    \mathbb E_{q_w}
    \left[
        \log \pi_E(a)
    \right]
    +
    \log Z(\alpha).
\end{align}

Since
\begin{equation}
    \frac{\partial \log Z(\alpha)}{\partial \alpha}
    =
    \mathbb E_{q_w}
    \left[
        \log \pi_E(a)
    \right],
\end{equation}
and
\begin{equation}
    \frac{\partial}{\partial \alpha}
    \mathbb E_{q_w}
    \left[
        \log \pi_E(a)
    \right]
    =
    \operatorname{Var}_{q_w}
    \left[
        \log \pi_E(a)
    \right],
\end{equation}
we obtain
\begin{equation}
    \frac{\partial H(q_w)}{\partial \alpha}
    =
    -\alpha
    \operatorname{Var}_{q_w}
    \left[
        \log \pi_E(a)
    \right]
    \le 0.
\end{equation}

Moreover,
\begin{equation}
    \frac{\partial \alpha(w)}{\partial w}
    =
    \beta-1
    >
    0.
\end{equation}
Therefore,
\begin{equation}
    \frac{\partial H(q_w)}{\partial w}
    =
    \frac{\partial H(q_w)}{\partial \alpha}
    \frac{\partial \alpha}{\partial w}
    \le 0.
\end{equation}
Hence, the entropy of the geometric-composition surrogate decreases
monotonically with the guidance coefficient $w$, which proves the theorem.

\end{proof}


\subsection{Policy Implementation details}\label{app:impl_details}

\subsubsection{Online RL Policy Implementation}

\paragraph{Exploration branch pre-training.}
We pretrain a separate exploration policy using the intrinsic reward produced by Random Network Distillation (RND). The RND prediction error provides a novelty signal, while the exploration policy generates actions. This pretraining is conducted separately from reward-based training, and its environment steps are excluded from the formal training budget. In the frozen-branch setting, the pretrained exploration policy remains fixed during formal training.

\paragraph{Classifier-free guidance sampling.}
At each reverse diffusion step, the exploitation and exploration policies predict noise from the same noisy action $x_t$. We combine their predictions as
\begin{equation}
    \epsilon_{\mathrm{cfg}}(x_t,t)
    =
    w\,\epsilon_{\mathrm{exploit}}(x_t,t)
    +
    (1-w)\,\epsilon_{\mathrm{explore}}(x_t,t),
\end{equation}
and use the combined prediction in the original DDPM reverse update. Guidance therefore acts throughout the generative trajectory, rather than interpolating two completed actions. 

\paragraph{Guidance coefficient.}
The coefficient $w\in[0,1]$ is a global scalar controlling the relative influence of the two branches: larger values favor exploitation, whereas smaller values favor RND-driven exploration. In the adaptive variant, $w$ is updated during training to improve estimated action value while maintaining a target divergence from the pure exploitation branch. 

For detailed implementation parameters, see Table \ref{tab:implementation_settings}.

\begin{table*}[ht]
\centering
\small
\caption{Network, optimization, and diffusion settings for Online RL.}
\label{tab:implementation_settings}
\begin{tabular}{@{}p{0.26\textwidth}p{0.30\textwidth}p{0.30\textwidth}@{}}
\toprule
Parameter & QVPO & QSM \\
\midrule
Policy hidden layers & $3\times256$, Mish & $2\times512$, Mish \\
Time embedding & 32-dim sinusoidal embedding & 64-dimensional learnable Fourier embedding with a $2\times128$ conditioning MLP \\
Critic & $3\times256$ hidden layers and 25 quantiles & $2\times512$ hidden layers  \\
RND hidden dimensions & $128/128$ & $128/128$ \\
RND pre-training steps & $100$k & $100$k \\
\midrule
Policy learning rate & $1\times10^{-4}$ & $3\times10^{-4}$ initially, with cosine decay \\
Critic learning rate & $3\times10^{-4}$ & $3\times10^{-4}$ \\
Batch size & 256 & 256 \\
Discount factor $\gamma$ & 0.99 & 0.99 \\
Target update coefficient $\tau$ & 0.005 & 0.005 \\
\midrule
Diffusion formulation & Discrete-time DDPM with noise prediction & Discrete-time DDPM with noise prediction \\
Reverse steps & 20 & 5 \\
Noise schedule & Cosine & VP ($b_{\min}=0.1$, $b_{\max}=10$) \\
Sampling noise scale & 1.0 & 1.0 \\
\bottomrule
\end{tabular}
\end{table*}

\subsubsection{RL Post-training Policy Implementation}

We initialize the base policy from a task-specific SFT checkpoint and keep
it frozen throughout RL post-training. After loading the SFT weights, we
copy its action decoder and action projections to create an independently
trainable exploration expert. Both experts share the frozen vision--language
backbone and its observation features, but maintain separate action-expert
parameters. Consequently, they produce the same velocity before RL updates.
At each denoising step, we combine their predictions as
\begin{equation}
v_{\mathrm{CFG}}
= (1-w)v_{\mathrm{base}}+w\,v_{\mathrm{exp}},
\end{equation}
where $w$ is the guidance scale. Only the exploration expert and value head
receive gradient updates.

We use the standard Flow-SDE and clipped PPO training procedure. During
action generation, one of the five denoising steps is randomly selected for
a stochastic Flow-SDE transition, while the remaining steps use ODE
transitions. PPO updates use the recorded transition log probability and
GAE advantages. The policy predicts 50 action steps, executes the first 25,
and then observes the environment again. Each run trains for 50 rounds and
is evaluated every 10 rounds on 64 fixed scenes. We test guidance scales
of $0.5$, $0.8$, and $0.98$; the original single-policy Flow-SDE + PPO
method is trained separately as the baseline.

For detailed implementation parameters, see Table \ref{tab:cfg_hyperparameters}.

\begin{table}[ht]
\centering
\small
\caption{Network and flow-sampling hyperparameters for CFG post-training.}
\label{tab:cfg_hyperparameters}
\begin{tabular}{lll}
\toprule
Parameter & $\pi_0$ & $\pi_{0.5}$ \\
\midrule
\multicolumn{3}{l}{\textit{Network and optimization}} \\
Visual inputs & 3 images & 3 images \\
Action dimension & 14 & 14 \\
Trainable actor & copied action expert & copied action expert \\
Critic input & action features (1024-D) & VLM features (2048-D) \\
Critic hidden widths & 512, 256, 128 & 512, 256, 128 \\
Actor learning rate & $5\times10^{-6}$ & $5\times10^{-6}$ \\
Critic learning rate & $10^{-4}$ & $10^{-4}$ \\
Global / micro batch size & 256 / 8 & 256 / 8 \\
\midrule
\multicolumn{3}{l}{\textit{Flow formulation and sampling}} \\
Action prediction horizon & 50 & 50 \\
Executed action chunk & 25 & 25 \\
Denoising steps & 5 & 5 \\
Flow-SDE noise level & 0.5 & 0.3 \\
Training transitions & 1 SDE + 4 ODE & 1 SDE + 4 ODE \\
Evaluation transitions & 5 ODE & 5 ODE \\
\bottomrule
\end{tabular}
\end{table}

\subsection{Cost comparison} \label{app:cost}

The runtime results are summarized in Table~\ref{tab:runtime_ablation}. Training time is measured over 1,000 environment steps. Evaluation time is averaged over five 1,000-step episodes.

\begin{table}[ht]
\centering
\caption{Runtime comparison on HalfCheetah-v4. }
\label{tab:runtime_ablation}
\begin{tabular}{llcc}
\toprule
Algorithm & Policy & Training (s) &  GPU memory (MiB) \\
\midrule
QSM  & Original         & 2.42  & 507  \\
QSM  & CFG-Policy       & 2.94  & 577  \\
QVPO & Original         & 8.87  & 831  \\
QVPO & CFG-Policy       & 9.37  & 867  \\
\bottomrule
\end{tabular}
\end{table}

\end{document}